%% file: log_2025.tex
\documentclass{article}
\usepackage{log_2025}				

\usepackage{caption}
\input{imports}

\usepackage{booktabs}						
\usepackage{multirow}						
\usepackage{amsfonts}						
\usepackage{amsthm}
\usepackage{graphicx}						
\usepackage{duckuments}						
\usepackage{placeins}
\usepackage{subcaption}
\usepackage[numbers,compress,sort]{natbib}	
\usepackage{xcolor} 

\newtheorem{theorem}{Theorem}

\newcommand{\id}{\iota}

\title[Short-Range Oversquashing]{Short-Range Oversquashing}

\author[Mishayev et al.]{%
Yaaqov Mishayev\\
Technion – Israel Institute of Technology\\
\email{yakov-m@campus.technion.ac.il}
\And
Yonatan Sverdlov\\
Technion – Israel Institute of Technology\\
\email{yonatans@campus.technion.ac.il}
\And
Tal Amir\\
Technion – Israel Institute of Technology\\
\email{talamir@technion.ac.il}
\And
Nadav Dym\\
Technion – Israel Institute of Technology\\
\email{nadavdym@technion.ac.il}
}

\begin{document}

\maketitle

\begin{abstract}
Message Passing Neural Networks (MPNNs) are widely used for learning on graphs, but their ability to process long-range information is limited by the phenomenon of \emph{oversquashing}. This limitation has led some researchers to advocate Graph Transformers as a better alternative, whereas others suggest that it can be mitigated within the MPNN framework, using virtual nodes or other rewiring techniques.

In this work, we demonstrate that oversquashing is not limited to long-range tasks, but can also arise in short-range problems. This observation allows us to disentangle two distinct mechanisms underlying oversquashing: (1) the \emph{bottleneck phenomenon}, which can arise even in low-range settings, and (2) the \emph{vanishing gradient} phenomenon, which is closely associated with long-range tasks. 

We further show that the short-range bottleneck effect is not captured by existing explanations for oversquashing, and that adding virtual nodes does not resolve it. In contrast, transformers do succeed in such tasks, positioning them as the more compelling solution to oversquashing, compared to specialized MPNNs.\footnote{Code available at \url{https://github.com/YakovM93/Short-Range-Oversquashing}.}
\end{abstract}

\section{Introduction}
Graph Neural Networks (GNNs) are the leading tool for learning on graph-structured data, with many of the most popular models falling into the category of \emph{Message Passing Neural Networks} (MPNNs). While MPNNs are computationally very efficient due to their ability to leverage graph sparsity, they are known to be successful only when using a small number of MPNN layers, typically 2--4. The difficulty in training deep MPNNs is commonly attributed to the phenomena of \emph{oversmoothing} \cite{oversmoothing} and \emph{oversquashing} \cite{alon2020bottleneck}.
Oversmoothing is the phenomenon in which, as the number of MPNN layers increases, node features become nearly indistinguishable from one another, with the extreme case often termed \emph{total collapse}. Our focus in this paper will primarily be on oversquashing.  

The term \emph{oversquashing}, coined by \citet{alon2020bottleneck}, refers to the difficulty of training MPNNs on \emph{long-range tasks}, that is, tasks that require communication between distant nodes to solve the problem accurately.
The authors explained that this  difficulty is caused by a \emph{bottleneck effect}, where the intermediate nodes on the path between two distant nodes need to have an increasingly large feature dimension in order to solve the problem by message passing. Later papers put more emphasis on low graph connectivity and vanishing gradients as the essential component leading to oversquashing. 

Oversmoothing and oversquashing are at the center of active discussion in the graph-learning community. On the one hand, many papers attempt to enable the use of MPNNs for long-range learning using techniques such as virtual nodes and other rewiring methods, which are aimed at reducing the range and difficulty of the learning problem. In a different direction, several papers have argued that \emph{Graph Transformers} (GTs) outperform MPNNs due to their ability to handle long-range tasks \cite{kreuzer2021rethinking,rampasek2022recipe,shirzad2023exphormer,dwivedi2022long}. Later results cast doubt on this claim, showing that with careful training, MPNNs with virtual nodes can obtain competitive results in many graph benchmarks, including Long Range Graph Benchmark (LRGB) \cite{dwivedi2022long}, which consists of tasks that arguably require long-range interactions. Similarly, it was argued in a recent work \cite{cai2023connection} that MPNNs with virtual nodes are able to simulate the attention mechanism: ``despite recent efforts, we still lack good benchmark datasets where GT can outperform MPNN by a large margin.''  

Our main goal in this paper is to advance the theoretical understanding of oversquashing by demonstrating that it can also arise in short-range tasks. Specifically, we construct a family of graph-learning problems that admit exact solutions with just two MPNN iterations, yet we prove that any MPNN must employ very large node-feature dimensions to solve them. Thus, these problems are affected by the bottleneck effect, even though the underlying graphs are well connected.
We also show empirically that the vanishing gradient problem does not occur for these problems.

In contrast, we show that popular synthetic long-range tasks considered in the literature suffer from vanishing gradients but not from the bottleneck effect. Thus, we claim that these are two distinct effects that were inadvertently mixed together.  

In addition, we show empirically that standard MPNNs, even when augmented with virtual nodes, perform poorly on these problems, whereas transformers solve them easily. This yields an interpretable test case in which MPNNs with virtual nodes are clearly outperformed by transformers.

\subsection{Related Work}
\paragraph{Oversquashing in graph neural networks}
The oversquashing phenomenon was first identified by \citet{alon2020bottleneck}, who demonstrated that MPNNs struggle to propagate information between distant nodes due to an exponential growth in the nodes' receptive field. They introduced the Tree Neighbors-Match problem, discussed in \cref{sec_oversquashing}, as a canonical example, where the bottleneck effect arises as the problem radius increases, arguing that the exponential growth of receptive fields with depth creates information bottlenecks at intermediate nodes.

\paragraph{Theoretical explanations of oversquashing}
Several theoretical frameworks have been proposed to explain oversquashing. \citet{topping2021understanding} connected oversquashing to the Ricci curvature of graph edges, proving that edges with negative curvature act as information bottlenecks and deriving bounds on gradient norms in terms of curvature. \citet{di2023over} provided direct gradient-decay analysis, deriving bounds on Jacobian norms that predict exponential decay when either the distance between nodes or the number of message-passing iterations is large.

\paragraph{Spectral properties and graph connectivity}
The \emph{spectral gap} (smallest nonzero eigenvalue $\lambda_1$ of the normalized graph Laplacian) has been proposed as a key indicator of oversquashing potential. When $\lambda_1$ is close to zero, the graph is nearly disconnected into multiple components, suggesting poor information flow. Several works have used spectral graph theory to predict and mitigate oversquashing, proposing rewiring strategies that increase $\lambda_1$ to improve information propagation \citet{topping2021understanding}.
\citet{black2023understanding} introduced \emph{effective resistance} as a measure for predicting oversquashing. For nodes $u$ and $v$, the effective resistance is $R_{u,v} = (1_u - 1_v)L^+(1_u - 1_v)$, where $L^+$ is the pseudoinverse of the graph Laplacian. This quantity, borrowed from electrical network theory, accounts for all paths between nodes and is proportional to commute time in random walks \cite{chandra1989electrical}. Black et al. proved upper bounds on gradient norms in terms of effective resistance, establishing connections between high resistance and vanishing gradients.

\paragraph{Graph Transformers and attention mechanisms}
Because self attention enables all-pairs communication in a single hop, Graph Transformers are often argued to mitigate oversquashing when paired with structural/positional encodings (e.g., Graphormer \cite{kreuzer2021rethinking}; GPS \cite{rampasek2022recipe}; Exphormer \cite{shirzad2023exphormer}). \citet{cai2023connection} showed theoretically that MPNN with virtual nodes (VN) can approximate self attention (including linear transformers) and, empirically, that strong MPNN+VN baselines are competitive on LRGB, sharpening the GT--MPNN comparison. \citet{rosenbluth2024distinguished} studied uniform expressivity and proved that GT and MPNN+VN are incomparable---neither subsumes the other---while much of ``universality'' in non-uniform settings stems from powerful positional encodings; their experiments report mixed outcomes across datasets. The Long-Range Graph Benchmark (LRGB) \cite{dwivedi2022long} was introduced to stress long-distance interactions, and subsequent re-evaluations with stronger baselines have narrowed parts of the once-reported transformer advantage \cite{tonshoff2023did}.

\paragraph{Oversmoothing and related phenomena}
Oversmoothing is another fundamental limitation of deep GNNs, in which node features become indistinguishable as the number of layers increases. \citet{oversmoothing} provide a comprehensive survey of oversmoothing, showing that it emerges as a consequence of repeated averaging operations in deep networks. Oversmoothing is inherently a deep-network phenomenon, requiring many layers before node features converge to similar values.

\paragraph{Information-theoretic perspectives}
\citet{alon2020bottleneck} introduced information-theoretic arguments to understand oversquashing, connecting information capacity requirements to exponentially growing receptive fields in deep networks. They argued that intermediate nodes must store information about an exponentially growing neighborhood, creating fundamental bottlenecks. This perspective has influenced subsequent work on understanding the theoretical limits of message-passing architectures and motivated the search for alternative architectures that avoid these bottlenecks.

More recently, \citet{arnaiz2025oversmoothing} published a broad position paper that critically examines common beliefs in graph machine learning, exposing conceptual ambiguities surrounding notions such as oversmoothing and oversquashing. Their work argues that many of these ideas have become intertwined in the literature and calls for clearer distinctions between them. This perspective complements ours by emphasizing the importance of separating computational bottlenecks (oversquashing) from topological assumptions, thereby motivating the more fine-grained theoretical and empirical analysis we undertake here.

\subsection{Notation and Preliminaries}
We begin by introducing notation. Graphs are denoted by $G=(V,E,\X)$, where $V$ is a finite set of nodes, $E$ is the set of graph edges, and $\X=(x_v)_{v\in V}$ denotes node feature vectors $x_v \in \mathbb{R}^d$. The set of neighbors of node $v$ is denoted by $\mathcal{N}_v$. Message-Passing Neural Networks (MPNNs) are graph neural networks that update each node's feature by combining its own feature with the features of its neighbors. Namely, the feature vector $h_v^k $ at each
layer $k$ is iteratively computed by
\begin{equation}\label{eq:mpnn}
    h^0_{v} = x_v, \quad  h^{k+1}_{v} = \phi_k(h^{k}_v,\psi_k(\{ h^{k}_{u}|u\in \mathcal{N}_v\})),
\end{equation}
where $\psi_k$ maps the multiset of neighboring node features in a permutation-invariant fashion to a vector, and $\phi_k $ maps pairs of vectors to a single vector.
Popular examples of MPNNs include GIN \cite{Xu:2019powerful} GAT \cite{Velickovic:2018we} GCN \cite{Kipf:2017tc} and many others \cite{Hamilton:2017tp,li:2016ggnn,ruiz:2020grnn}.

\section{Oversquashing: Long Range and Short Range}
\label{sec_oversquashing}
%
\begin{figure}[t]
	\centering
	\includegraphics[width=\linewidth]{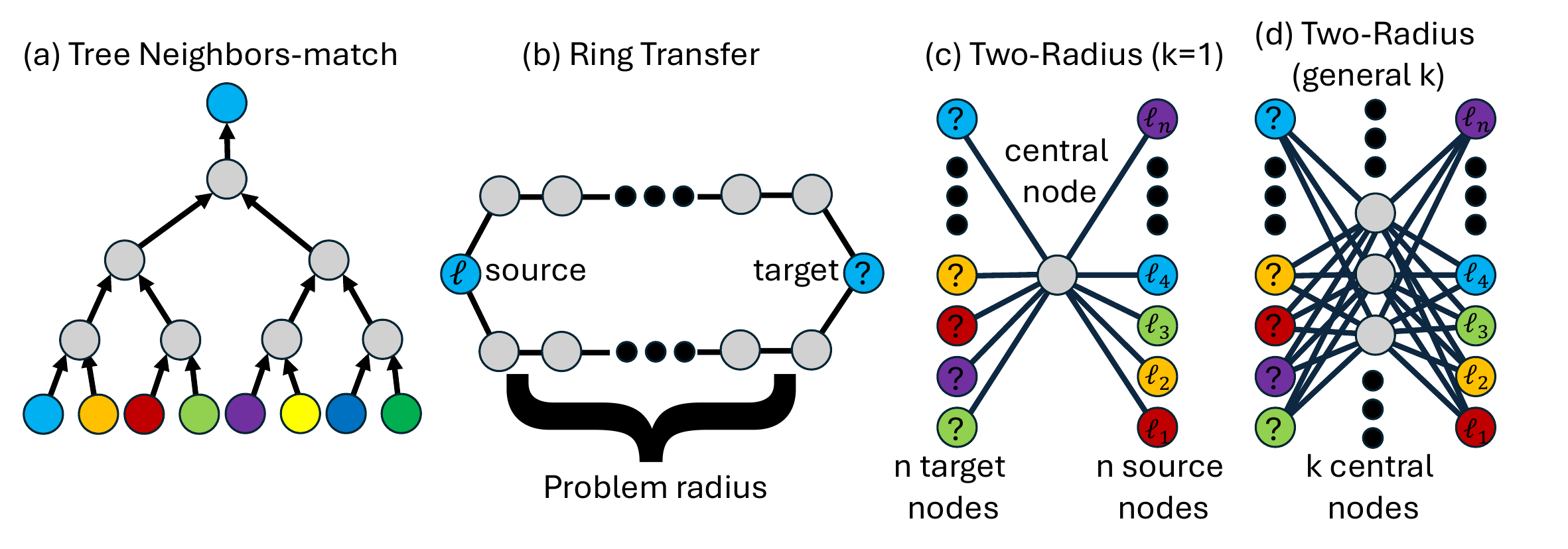}
    \caption{
    Illustration of synthetic graph-transfer problems. 
    (a) \textbf{Tree Neighbors-Match}: information is transferred from leaves to a target node through a tree of depth~$r$. 
    (b) \textbf{Ring Transfer}: a source and target are connected by two disjoint paths of length~$r$. 
    (c) \textbf{Two-Radius}: $n$ sources, $n$ targets, and a single central node. 
    (d) \textbf{Generalized Two-Radius}: $k$ central nodes. 
    Node colors represent source and target identifiers; gray denotes central nodes.}
    \label{fig:problem}
\end{figure}
To improve the theoretical understanding of oversquashing, we study a family of synthetic graph-transfer problems. We begin by introducing terminology and notation that will be used throughout the examples.

We consider graphs whose node set is a disjoint union $V=S\cup C \cup T $, where $S$ denotes \emph{source nodes}, $C$ denotes \emph{central nodes}, and $T$ denotes \emph{target nodes}. In essence, the goal of these tasks is to transfer information from source nodes to target nodes, with the central nodes serving solely to conduct that information.

Each node feature is a pair \(x_v = (\id_v, \ell_v)\),  
where \(\id_v \in \{0,1,\ldots,n\}\) is a \emph{node identifier}, and  
\(\ell_v \in \{1,\ldots,L\}\) is a \emph{node label}. 
Source nodes have unique identifiers $\id_v \in \{1\ldots,n\}$, and their labels represent information to be transferred. 
Central nodes are all assigned the identifier $\id_v=0$. The identifiers of target nodes specify from which source they should receive information (see \cref{fig:problem}). The identifiers and labels are encoded as one-hot vectors.


We begin with two well-known problems from the graph learning literature. 
The first problem, \textbf{Tree Neighbors-Match}, was introduced by \citet{alon2020bottleneck}. 
We consider a binary tree, whose source nodes are its leaves, each assigned a distinct identifier $\id_s$ and a label $\ell_s$. The root of the tree is connected to a target node $t$, which is assigned an identifier $\id_t$ (see \cref{fig:problem}(a)). The goal is to assign to the target node the label of the source node that has the same node identifier. Namely, the MPNN needs to find the leaf node $s$ for which $\id_s=\id_t $, and set the output feature $h^{K}_{t}$ of node $t$ to $\ell_s$.  
 
Surprisingly, \citet{alon2020bottleneck} demonstrated empirically that, as the depth of the tree increases, standard MPNNs struggle to solve this seemingly simple task. They attribute this to the exponential growth in the number of leaves with the depth. Since message passing aggregates information locally, solving the task perfectly requires the root node to encode the information from all leaves, which in turn demands a vector of very high dimension---rendering the approach impractical.
 
 
The second problem, \textbf{Ring Transfer}, is a simple graph-transfer task introduced by \citet{bodnar2021weisfeiler} and further studied by \citet{di2023over}. Here, a source node $s$ and a target node $t$ are connected by two paths of length $r$ (see \cref{fig:problem}(b)). The goal is to transfer the label $\ell_s$ from the source to the target. \citet{di2023over} showed that this task also poses difficulties for MPNNs. Their analysis focuses on vanishing gradients rather than on the bottleneck phenomenon. Our first theoretical result, stated below, confirms this intuition: the Ring Transfer task indeed requires long-range interaction and is therefore prone to vanishing gradients. However, it does \emph{not} suffer from the bottleneck effect, in the sense that it does not require high-dimensional node features.

\begin{theorem}\label{thm:constant_dim}
For any $r \geq 1$, the Ring Transfer task with radius $r$ requires at least $r$ iterations of an MPNN. 
However, there exists an MPNN that solves the task exactly whose node feature dimension is independent of $r$. This also holds if the ring topology is replaced with any other graph. 
\end{theorem} 

\begin{proof}[proof idea]
The necessity of at least $r$ iterations is intuitive and well known. Intuitively, a constant feature dimension is sufficient because all that is needed is to recursively transfer the input source feature vector to neighboring nodes until the target node is reached. For a formal proof, see \cref{app:proofs}.
\end{proof}
 
\subsection{The Two-Radius Problem}
We now introduce a new synthetic graph-transfer task, the \textbf{Two-Radius problem}. We show that, although it is solvable in theory with only two MPNN iterations, it nevertheless suffers from the bottleneck phenomenon. We further demonstrate that MPNNs struggle to solve it in practice.

We first consider a simple variant of the problem---a family of graphs denoted by $\gG_n$, with $n \geq 1$ (see \cref{fig:problem}(c)). 
Each graph $\gG_n = (V,E,\X)$ has a vertex set $V = S \cup C \cup T$ consisting of $n$ source nodes, $n$ target nodes, and a single central node. 
The source nodes $s \in S$ are assigned distinct identifiers $\id_s \in \{1,\ldots,n\}$, and the $n$ target nodes $t\in T$ are assigned the same set of identifiers. 
Each source node is also assigned a distinct label $\ell_s \in \{1,\ldots,n\}$, not necessarily identical to its identifier.
The goal is to construct an MPNN such that after $K$ iterations, the output features $h_t^K$ of the target nodes satisfy
\[
h_t^K = \ell_s \quad \text{whenever } \id_t = \id_s \, .
\]

As we show below, this problem can be solved exactly by an MPNN, but only at the cost of a very high feature dimension, of order $n \log n$. 
This is perhaps unsurprising, since the graphs under consideration are nearly  disconnected: removing the single central node disconnects the graph and renders the task impossible to solve by an MPNN. 
Nonetheless, in the next section, we show that many of the measures proposed in the MPNN literature to assess connectivity and predict oversquashing fail to identify this graph as problematic. 
Moreover, the graphs in $\gG_n$ can be made much better-connected \emph{without} resolving the bottleneck issue.

To show this, we consider a more general family of graphs $\gG_{n,k}$ with $k$ central nodes (see \cref{fig:problem}(d)). 
Each central node is connected to all source and target nodes. 
Due to the permutation invariance of MPNNs, the central nodes are indistinguishable. 
As a result, adding more central nodes does not resolve the bottleneck phenomenon, even though it substantially improves the connectivity of the graph. 
We formalize this in the following theorem.

\begin{theorem}\label{thm_two_radius}
There exists an MPNN with $T=2$ iterations that exactly solves the transfer task on $\gG_{n,k}$. However, when using $b$-bit floating-point arithmetic, any MPNN that solves the transfer task on $\gG_{n,k}$ with $T$ iterations and intermediate node features of dimension $d_t$ must satisfy
$$\sum_{t=1}^T d_t \geq \frac{n}{2b}\log_2(n/2) $$ 
for every central node $c \in C$.
\end{theorem}
\begin{proof}
Fix some ordering of the source and target node identifiers, and some initial labeling, leading to some annotated graph $G \in \gG_{n,k}$. Next, for any permutation $\tau \in S_n$, consider the new problem instance obtained by permuting the labels by $\tau$ while leaving the source and target nodes fixed, giving a new graph $G_\tau $. Let   $x_c^t(\tau)$ denote the node features in $c$ after $t$ MPNN iterations applied to $G_\tau$, and denote by $\mathbf{t}(\tau)$ the vector of all target nodes $\mathbf{v}=(x_v, v\in T) $ obtained after $T$ MPNN iterations applied to $G_\tau$. Then $\mathbf{v}(\tau) $ is a function of the central nodes 
$$\mathbf{v}(\tau) =\mathbf{v}(x_c^t(\tau), t=1,\ldots,T)  $$
where $c$ is any fixed central node (here we use the fact that the nodes at all central nodes are the same, so we can look only at one of them). Since the MPNN solves the task exactly, we know in particular that $\mathbf{v}(\tau)\neq \mathbf{v}(\sigma)$ for any two distinct permutations $\tau\neq \sigma$. Therefore $\mathbf{v}$ can assume $n!$ different values as the permutation $\tau$ changes, and therefore, the  vector $(x_c^t(\tau), t=1,\ldots,T) $ can also attain $n!$ different values. For this to be possible, this  vector must of dimension sufficient to contain so many values. Namely the total dimension $d=\sum_{t=1}^T d_t$ of the vector must satisfy $(2^b)^d\geq n! $, which implies that
 $$ b\cdot d \geq \log_2(n!)\geq \log_2\left( \left(\frac{n}{2}\right)^{n/2} \right)=\frac{n}{2}\log_2(n/2) $$
\end{proof}



As a concluding remark, we note that the high feature dimension required by \cref{thm_two_radius} is not the only difficulty in the problem. As our experiments below show,  MPNNs indeed struggle to solve this problem even with high feature dimension. We believe that the main difficulty is the challenge of reliably mapping an input multiset of increasingly large size $n$ of node features into the intermediate representation carried by the central nodes, without incurring significant distortion. The dependence of the distortion of a multiset map on the problem size is discussed, e.g., in \cite[Theorem 3.3]{davidson2025on}.

\subsection{Empirical Performance of MPNNs on the Two-Radius Problem}
We next evaluate empirically whether the bottleneck phenomenon in the Two-Radius problem indeed leads to practical performance degradation. 
Specifically, we consider the case $k=1$ with $n \in \{10, 50, 150, 200\}$, using standard MPNN architectures: GCN \cite{Kipf:2017tc}, GIN \cite{Xu:2019powerful}, GAT \cite{Velickovic:2018we}, and GraphSAGE \cite{Hamilton:2017tp}. 
We evaluate all methods with feature dimensions of 256 and 1024 and three different learning rates, and report, for each method, the best result obtained across these runs. 
In addition, we evaluate a simple Set Transformer \cite{Lee:2019set} without augmenting with structural/positional encoding. The Set Transformer treats the vertex features as a multiset, while ignoring the edge structure of the graph.
Since Transformer allows pairwise interactions between all nodes, it is not expected to suffer from oversquashing. 
We further include a standard MLP, which, unlike the other methods, is not permutation invariant, and is thus a priori not expected to suffer from oversquashing, although it may lack sufficient inductive bias to generalize well on this task.


The experiment results appear in \cref{fig:acc_vs_n_dim}. 
As seen in the figure, MLPs perform poorly on this task even for small $n$.
MPNN performance degrades as $n$ increases, in line with our analysis. 
The only methods that succeed are the Transformer, which consistently achieves 100\% accuracy on all instances, and GAT, which, with a high embedding dimension, can reach 90\% accuracy even when $n=200$. 
In addition to its lower accuracy compared to the Transformer, GAT was substantially more difficult to train: these results required using a very low learning rate and training for many epochs. 
As shown in \cref{fig:epochs_and_k}(a), achieving at least 92\% accuracy required more than 100 epochs with GAT, whereas the Transformer needed fewer than 10. A more fine-grained evaluation of the effect of hidden-feature dimension appears in \cref{fig_model_accuracy_vs_dimension} in \cref{appendix_additional_experiment_results}.

\begin{figure}[!htbp]
    \centering
    \includegraphics[width=1\linewidth]{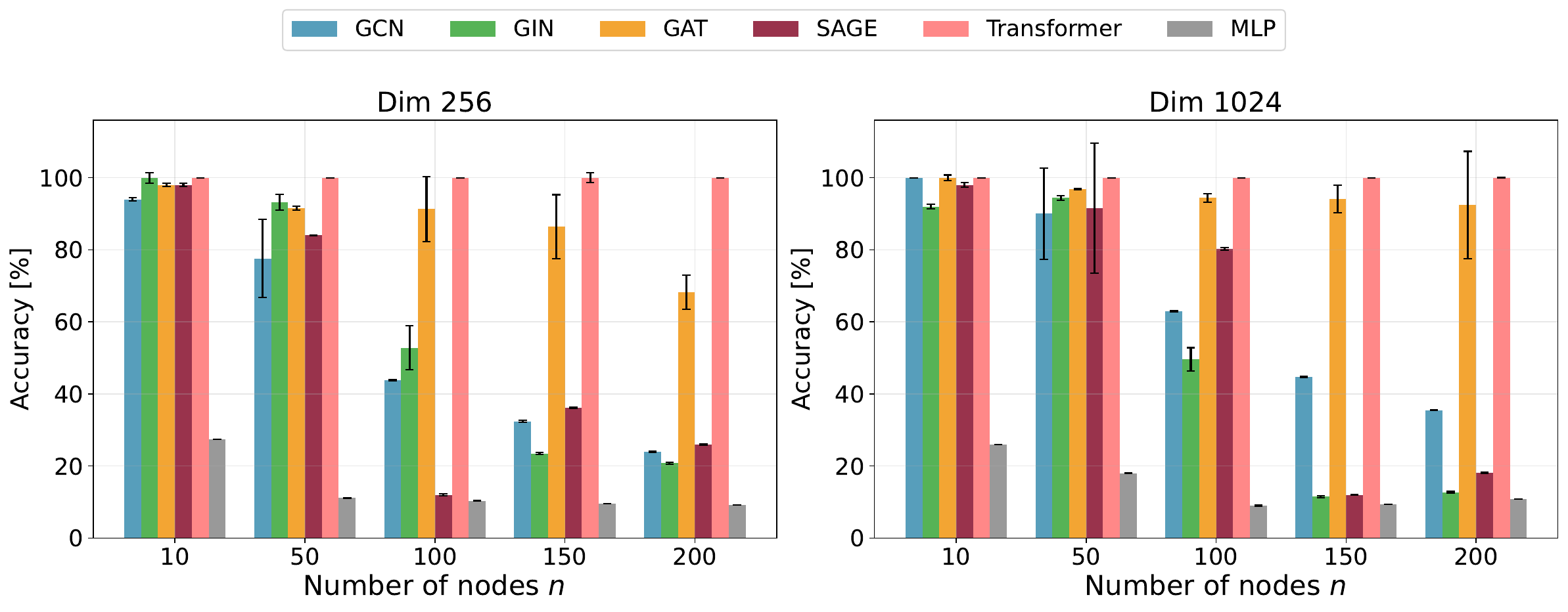}
        \caption{Test accuracy comparison across different models on the Two-Radius problem. Performance is evaluated for varying numbers of nodes $n \in \{10, 50, 150, 200\}$ with hidden dimensions of 256 and 1024. Transformer consistently achieves 100\% accuracy while MPNN performance degrades as $n$ increases. Error bars indicate the standard error of the mean.}
    \label{fig:acc_vs_n_dim}
\end{figure}

\FloatBarrier 

Next, we examine the effect of changing the number of central nodes $k$ on the performance of MPNNs. 
As noted earlier, increasing $k$ improves graph connectivity, but it is not expected to improve MPNN performance, because permutation equivariance implies that all central node features are identical across the message-passing process.
Empirically, running GCN on the Two-Radius problem with varying $k$, we found that performance indeed does not improve as $k$ grows and, perhaps surprisingly, is even worse than with $k=1$. 
This is shown in \cref{fig:epochs_and_k}(b).\\

\FloatBarrier 

\begin{figure}[!htbp]
    \centering
    \begin{minipage}[t]{0.48\linewidth}
        \centering
        \includegraphics[width=\linewidth]{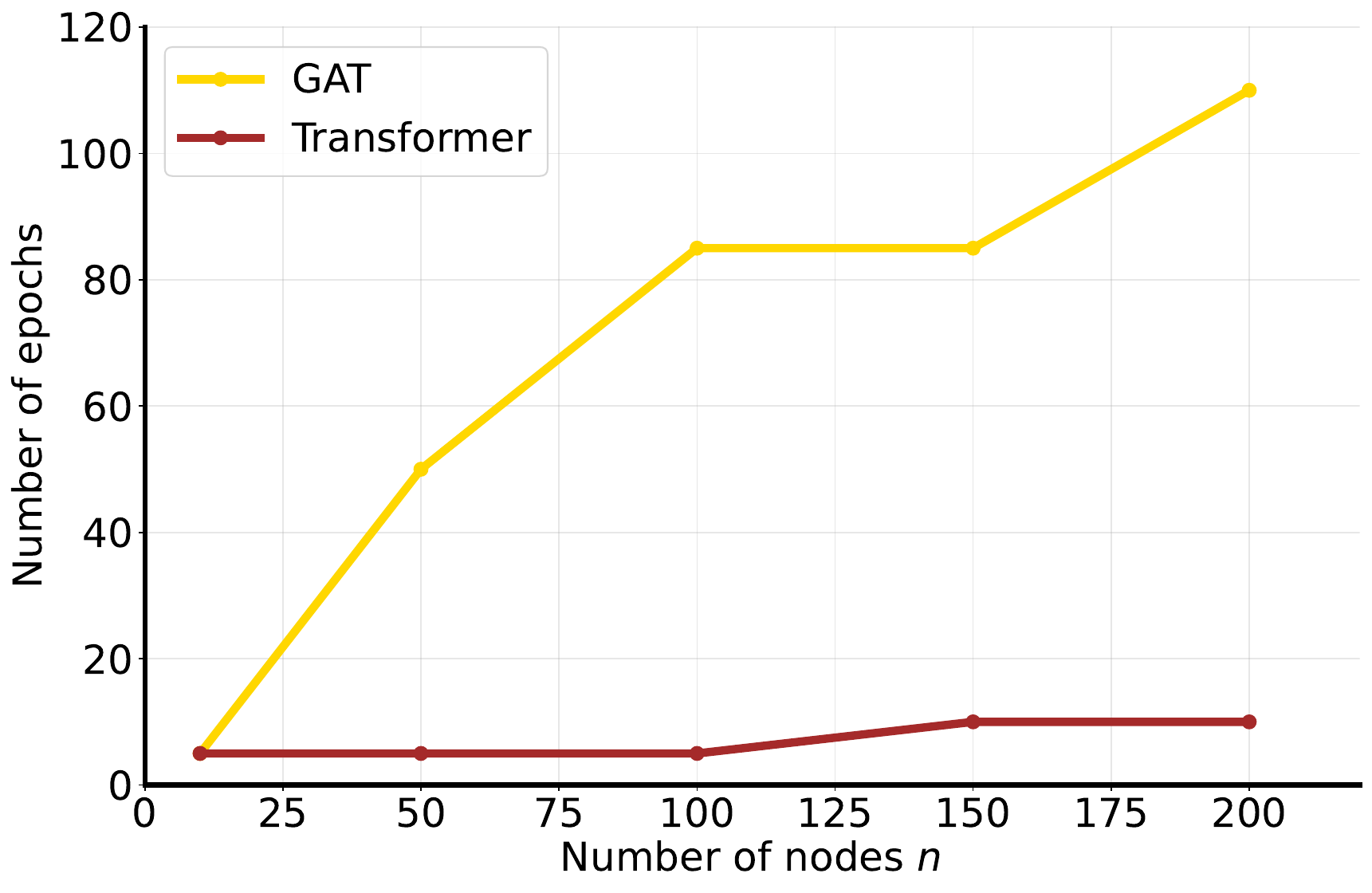}
        \vspace{2pt}{\small{(a)}}
    \end{minipage}\hfill
    \begin{minipage}[t]{0.48\linewidth}
        \centering
        \includegraphics[width=\linewidth]{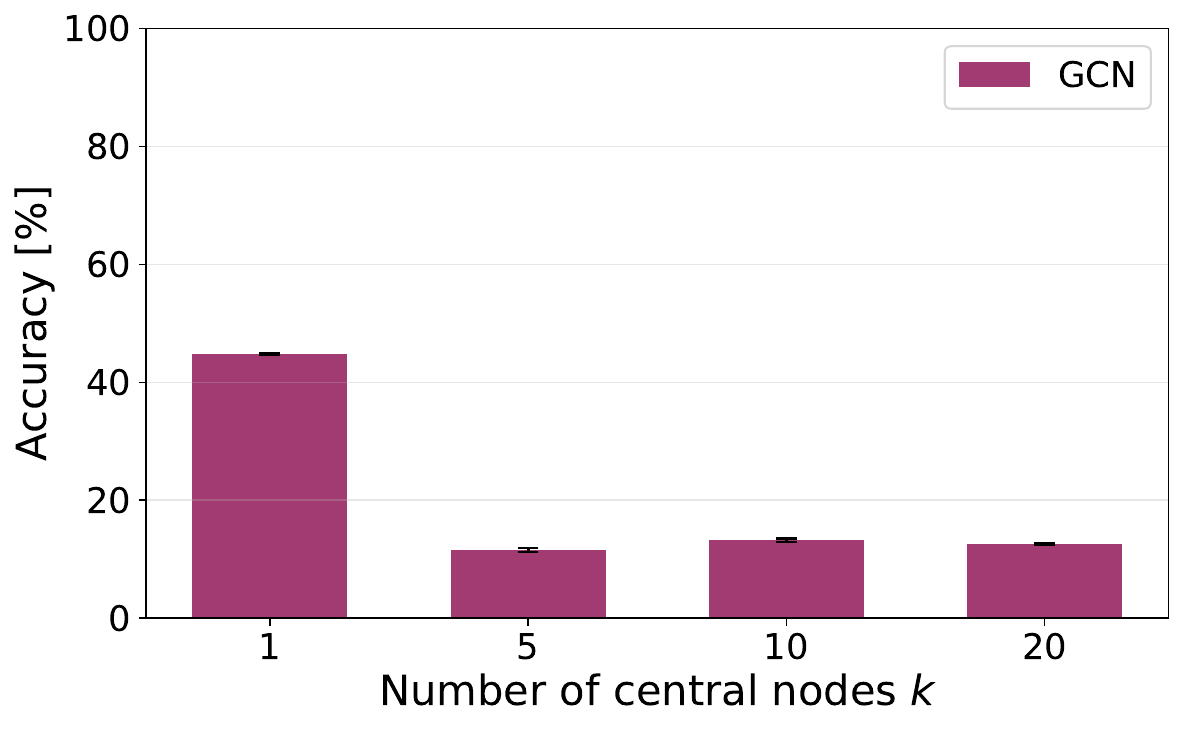}
        \vspace{2pt}{\small{(b)}}
    \end{minipage}
    \caption{(a) Training efficiency comparison between GAT and Transformer on the Two-Radius problem, measured in the number of epochs required to achieve 92\% accuracy. (b) Effect of the number of central nodes $k$ on GCN performance for the Two-Radius problem with $n=100$. Accuracy remains poor regardless of $k$, demonstrating that increasing graph connectivity via additional central nodes does not resolve the bottleneck phenomenon. Error bars in (b) (barely visible due to low variance) indicate the standard error of the mean.}
    \label{fig:epochs_and_k}
\end{figure}

In summary, we find that MPNNs struggle to solve the Two-Radius problem, whereas Set Transformer succeeds easily. 
This strengthens the case for Graph Transformers over MPNNs, and serves as an example of a scenario where disregarding the graph structure can be advantageous---a phenomenon discussed in \cite{bechler2024graph}.

\section{Oversquashing Measures Don't Explain the Two-Radius Problem}
%

As seen in the previous section, the Two-Radius problem suffers from the bottleneck effect (\cref{thm_two_radius}) and MPNNs struggle to solve it in practice. 
The goal of this section is to revisit common proposals for measuring and characterizing oversquashing, and to show that most of these measures fail to capture the oversquashing that arises in the Two-Radius problem.

\subsection{Current Explanations for Oversquashing}

\paragraph{Problem Radius}
In the original oversquashing paper \cite{alon2020bottleneck}, the authors argued that the root cause of oversquashing is a large \emph{problem radius}. Their reasoning connects several factors: the literal excessive compression of information in node features (oversquashing), the number of message-passing iterations, and, consequently, both the problem radius and the growth rate of the receptive field of message-recipient nodes, which are, respectively, causes and consequences of the number of MPNN iterations. 

In the Two-Radius problem introduced in the previous subsection, the problem radius is fixed at $2$, independent of $n$, and after only two message-passing iterations, the receptive field of all nodes is constant. This setting allows us to \emph{decouple} the effect of large problem radius, which leads to vanishing gradients, from the bottleneck effect, which arises from large receptive fields. 
We regard these as two distinct mechanisms underlying oversquashing.


\paragraph{Vanishing Gradients}
A common explanation of oversquashing is that long range leads to vanishing gradients. In fact, many explanations of oversquashing provide analysis showing how the size of the gradients depends on certain topological properties of the graph, and using this analysis to devise rewiring techniques aimed at improving these topological properties and reducing oversquashing. 

The presence of vanishing gradients in the Two-Radius and Ring Transfer problem was evaluated by computing the average gradient norm of the model output (GCN) over 10 randomly selected test samples. The results are shown in \cref{fig:vanishing}. For the Ring Transfer problem, both accuracy and gradient norm decrease drastically with the problem radius, whereas for the Two-Radius problem, as $n$ increases, the accuracy of MPNN decreases, but its gradient norm exhibits only minor decrease. This suggests that the vanishing gradient effect is related to long-range interactions, but not to the bottleneck effect. To ensure that the hindered performance in the Two-Radius problem is not caused by oversmoothing, we computed the relative Mean Absolute Difference (MAD) energy \cite{chen2020measuring} over the target nodes (see \cref{appendix_experimental_details} for details). As seen in the figure, this energy does not decay toward zero, thereby ruling out oversmoothing as the underlying cause.\\

\begin{figure}[h]
    \centering
    \includegraphics[width=1\linewidth]{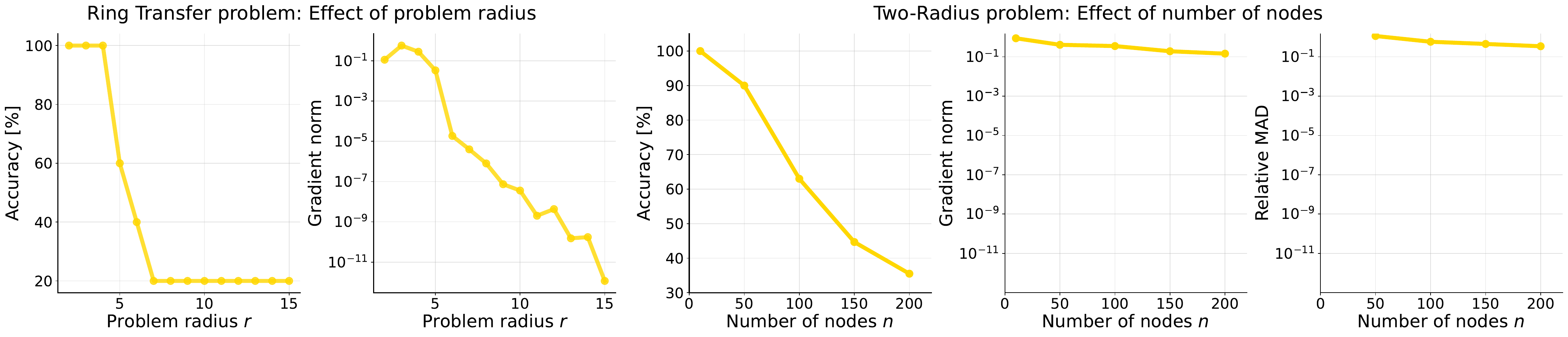}
    \caption{GCN's accuracy deteriorates as the problem radius increases,  which is correlated with  vanishing gradients. On the other hand, problems with bottlenecks are also difficult for MPNNs, but they do not suffer from vanishing gradients or oversquashing.}
    \label{fig:vanishing}
\end{figure}

\paragraph{Spectral methods}
Several works have proposed spectral quantities of the graph Laplacian as indicators of oversquashing, in particular the \emph{spectral gap} (the difference between the first and second eigenvalues of the normalized Laplacian) and the closely related \emph{Cheeger constant}, which measures the presence of sparse cuts in the graph \cite{topping2021understanding,banerjee2022oversquashing,karhadkar2022fosr}.  
Intuitively, graphs with a small spectral gap or low Cheeger constant contain bottlenecks that may hinder information flow, and such properties have been argued to correlate with oversquashing.  
However, in the Two-Radius problem with a single central node,  the Cheeger constant is equal to $1$ for all values of $n$. Moreover, when $k\geq 1$ central nodes are taken, the graph becomes more connected, and the Cheeger constant grows with $k$:
\begin{theorem}\label{thm:cheeger}
The Cheeger constant $h_G$ for the graph $\gG_{n,k}$ with $k\leq 2n$ central nodes is lower-bounded by $h_G \geq \frac{k}{8} .$  
\end{theorem}
Nevertheless, the empirical and theoretical difficulty of the problem clearly increases with $n$, as shown in \cref{fig:acc_vs_n_dim} and \cref{thm_two_radius}, showing that these spectral criteria fail to capture the bottleneck effect responsible for oversquashing in this setting.


%
\paragraph{Ricci Curvature}
The concept of Ricci curvature as an explanation for oversquashing was introduced in \citet{topping2021understanding}. In that work, the authors proved upper bounds on certain gradient norms in terms of this curvature, thereby showing that edges with negative curvature act as ``information bottlenecks,'' which can lead to vanishing gradients. 

Specifically, their main relevant result (Theorem~4) states that for any pair of vertices \(i,j\) connected by an edge, and for any \(\delta > 0\), under some additional assumptions, there exists a nonempty set \(Q_j\) of vertices in the two-hop neighborhood of \(i\) such that
\[
    \frac{1}{|Q_j|} \sum_{k \in Q_j} \left| \frac{\partial h^2_k}{\partial h^0_i} \right|
    < (\alpha \cdot \beta)^2 \cdot \delta^{\tfrac14},
\]
where \(\alpha\) and \(\beta\) are upper Lipschitz bounds on the update and aggregation functions in the message-passing procedure, and \(\ell_0\) is any integer between \(0\) and \(L-2\), with \(L\) denoting the number of message-passing iterations. Their result requires
\[
    \mathrm{Ric}(i,j) + 2 \leq \delta < \max\{d_i, d_j\}^{-\tfrac12},
\]
where \(\mathrm{Ric}(i,j)\) is the \emph{Ricci} curvature of edge \((i,j)\), also referred to as the \emph{balanced Forman curvature}.

In the Two-Radius problem, however, every edge \((i,j)\) connects a vertex of degree \(1\) to a vertex of degree $n$. This implies \(\mathrm{Ric}(i,j) = 0\) (see Definition~1 in \cite{topping2021understanding}), and the above assumption becomes
\[
    2 \leq \delta < n^{-\tfrac{1}{2}} \leq 1,
\]
which is impossible. Hence, their result is not applicable to this problem and, in particular, does not predict oversquashing.

\paragraph{Effective Resistance} 
A closely related quantity, proposed by \citet{black2023understanding}, is the \emph{effective resistance}.  Inspired by the concept of resistance in electrical networks, the effective resistance between two nodes $u$ and $v$ decreases as the number of paths between them increases, and increases with the lengths of those paths.
Formally it is defined as
\begin{equation}
    R_{u,v} = \left(1_u-1_v\right) L^+ \left(1_u-1_v\right),
\end{equation}
where $L^+$ is the pseudo-inverse of the non-normalized graph Laplacian. In the special case where $u$ and $v$ are connected by vertex-disjoint paths, the effective resistance can be expressed by a simpler formula (see Figure~1 therein)
\begin{equation}\label{eq_effective_resistance_simple}
    R_{u,v} = {\left( \sum_{\textup{$p$ is a path from $u$ to $v$}} \textup{Length}(p)^{-1} \right)}^{-1}.
\end{equation}
%

To show that high effective resistance leads to oversquashing bottlenecks, the authors proved an upper bound on the Frobenius norm of the Jacobian in terms of the effective resistance between the nodes. 

However, in our example, it can be seen from \cref{eq_effective_resistance_simple} that the effective resistance between the source and destination vertices always equals $2$, independently of $n$, whereas the core phenomenon of oversquashing is clearly aggravated as $n$ increases.

%
\paragraph{Direct gradient-decay analysis} 
\citet{di2023over} further developed the ideas of \citet{black2023understanding} and derived upper bounds on the Jacobian norm directly from basic topological features of the input graph, that is, without passing through spectral graph theory. Their first result, Theorem 3.2, is of the form
\begin{equation}\label{giovanni_first_result}
    \lVert \frac{\partial h_v^{(m)}}{\partial h_u^{(0)}} \rVert
    \leq
    C^m \left(S^m\right)_{u,v},
\end{equation}
where $C$ is a constant depending on the model and hidden feature dimension, and $S$ is a matrix derived from the model parameters and the graph adjacency matrix. When the right-hand side of \cref{giovanni_first_result} decays exponentially with $m$, their theorem predicts exponential decay of the Jacobian norm. However, in our setting, the number of message-passing iterations is constant $m=2$, so the theorem does not predict exponential gradient decay.

In another result \cite[Theorem 4.1]{di2023over}, they present a bound that similarly predicts exponential gradient decay, with the exponent depending on the distance between the two nodes and the number of message-passing iterations. Since both quantities in our example are constant, this theorem does not predict the observed oversquashing.

The remaining results in \cite{di2023over} similarly assume either large distances or large number of message-passing iterations, both of which do not hold in our example.

In \cref{fig:problem}(b), it can be seen that for the Ring Transfer problem, the vanishing-gradient bounds do capture the oversquashing effect demonstrated empirically in \cref{fig:vanishing}.

%

\section{Virtual Nodes and Transformers}
Several papers have recently argued, both theoretically \cite{cai2023connection,qian2024probabilistic,rosenbluth2024distinguished,southern2024understanding} and empirically \cite{pham2017graph,qian2024probabilistic,hwang2024analysis}, that MPNNs with virtual nodes (VNs) can effectively handle oversquashing. 
In the context of long-range oversquashing, this is indeed reasonable, since virtual nodes reduce the effective problem radius to $2$. 
However, for the Two-Radius problem, a virtual node would act similarly to the existing central nodes in these graphs, so adding a virtual node is not expected to yield substantial improvement for MPNNs.

To evaluate this hypothesis, we compared GCN with and without a virtual node, as shown in \cref{fig:VN}. 
As seen in the figure, while adding a virtual node does seem to improve performance for some values of $n$, GCN+VN still fails to solve the Two-Radius problem for larger values of $n$. 
Thus, the Two-Radius problem provides a synthetic example in which MPNN+VN is less effective than Transformers.

\begin{figure}[t]
    \centering
    \includegraphics[width=0.5\linewidth]{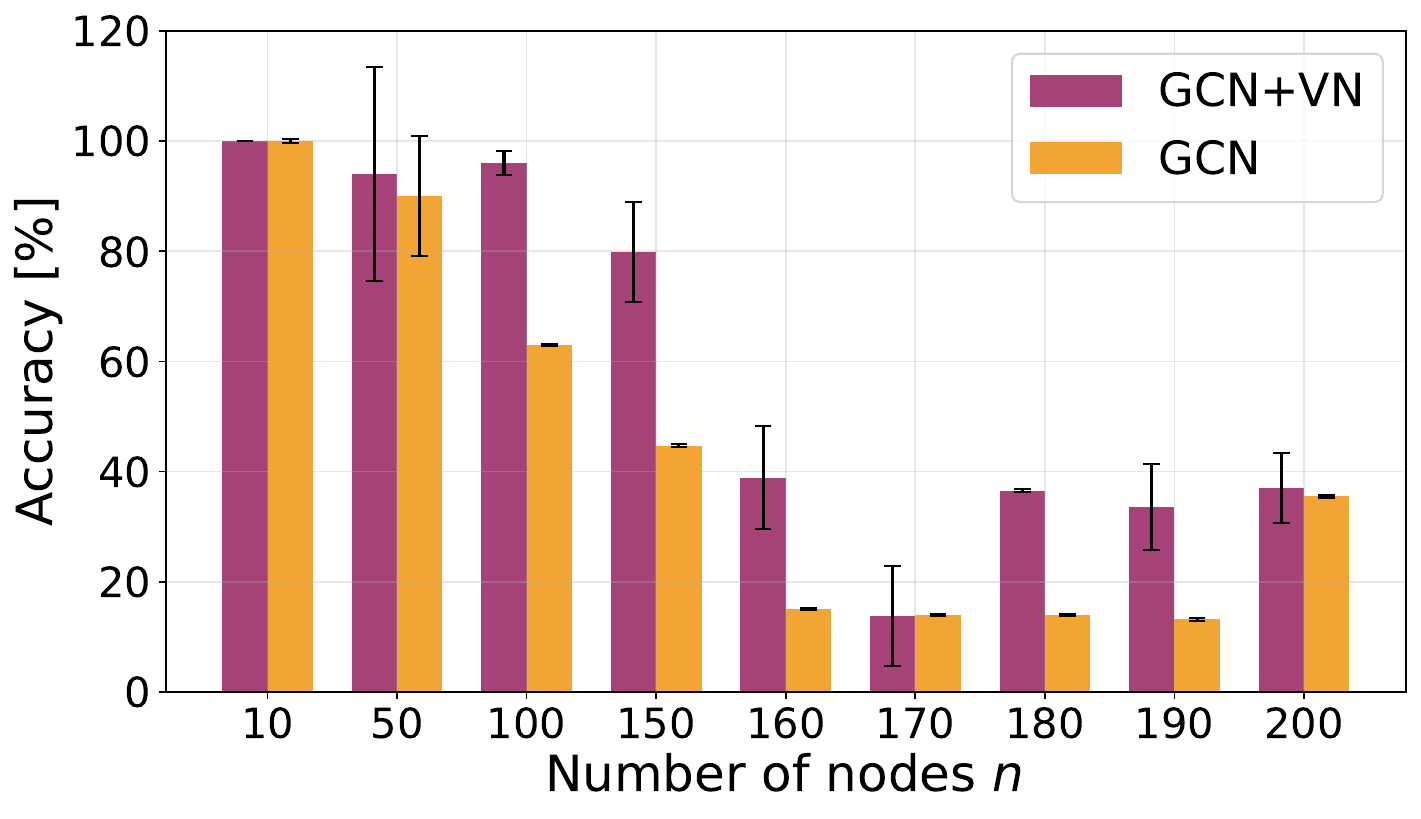}
    \caption{Comparison of GCN performance with and without virtual nodes (VN) on the Two-Radius problem. While virtual nodes provide modest improvements, performance still degrades significantly as $n$ increases, indicating that VNs do not fully address the bottleneck in short-range oversquashing. Error bars indicate the standard error of the mean.}
    \label{fig:VN}
\end{figure}

\section{Conclusion}
In this work, we revisited the phenomenon of oversquashing in MPNNs and demonstrated that it is not restricted to long-range tasks. 
Through the Two-Radius problem, we identified a setting where oversquashing arises even in short-range scenarios and showed that this corresponds to a bottleneck effect, separate from the vanishing gradient effect that dominates in long-range tasks. 

Our theoretical analysis established that solving the Two-Radius problem with MPNNs requires large feature dimensions that depend on the graph size, and our empirical results confirmed that standard MPNNs, even when augmented with virtual nodes, struggle on this task.
By contrast, Graph Transformers and related architectures succeed, underscoring their potential as a more robust solution to oversquashing in settings where MPNNs remain bottlenecked.

Our results also clarify the limitations of existing measures of connectivity as predictors of oversquashing, and highlight the need for refined metrics that account for short-range bottlenecks. 
We hope that the Two-Radius problem will serve as a useful benchmark for oversquashing in future studies, and that our analysis can guide the design of new architectures that combine the efficiency of MPNNs with the expressivity of transformers.

\bibliographystyle{unsrtnat}
\bibliography{reference}
\appendix
\section{Proofs}\label{app:proofs}
\begin{proof}[Proof of Theorem \ref{thm:constant_dim}]
The setting we are considering is that we are given some graph $G$, with a single source node $s$ with node features $x_s=(1,\ell_s) $, where $1$ is the node identifier, signifying that this is the source node, and $\ell_s$ is the source label coming from some finite alphabet $\{1,\ldots,A\}$. All other nodes $v$ of the graph are given an initial value $x_v=(0,\ell_v) $, where $\ell_v$ comes from the same alphabet (and is irrelevant to the problem at hand). There is one node, the target node $t$, whose shortest path distance from the node is $r$, and the goal is for an MPNN to update the target node to achieve the original value of the source node, namely $h_t^k=x_s $. 

It is clear that in less that $r$ iterations, an MPNN will not be able to achieve this goal. However, in $r$ iterations this goal is easily achieved by an MPNN whose feature dimension does not depend on $r$. Namely, we begin as always with the initial features  $h_v^0=x_v$, and iteratively define $h_v^{k+1}\in \mathbb{R}^2$ by 
$$ h_v^{k+1}[1]=\max_{u\in \mathcal{N}_v\cup \{v\}} h_u^{k}[1], \quad h_v^{k+1}[2]=\max_{u\in  \mathcal{N}_v\cup \{v\}} h_u^{k}[1]\cdot h_u^{k}[2] $$
We can then prove recursively on $k$ that if the distance of $v$ from $s$ is $\leq k$, then $h_v^k=h_s^0 $, and otherwise $h_v^k[1]=0 $. 

For $k=0$ this is clearly true. If the claim is true for $k$, then for $k+1$ we have that, if the distance of $v$ from $s$ is more than $k+1$, than all its neighbors $u$ will all have distance more than $k$ from $s$, and so  $h_u^k[1]=0 $  and therefore $h_v^{k+1}[1]=0 $. If $v$ is a node of distance $\leq k+1$ from $s$, then there is some $u\in N(v)\cup \{v\}$ whose distance from $s$ is $\leq k$. As a result $h_v^{k+1}[1]=1=h_s^0, h_v^{k+1}[2]=h_u^k[2]=h_s^0[2] $. 

In particular, it follows that after $r$ iterations the target node will obtain the value of the source node. This concludes the proof. 
\end{proof}

\begin{proof}[Proof of Theorem \ref{thm:cheeger}]
Recall that for a graph $G$, and a subset $A\subseteq V$, the boundary $\partial A$ is defined as the number of edges between nodes in $A$ and nodes in the complement of $A$, and the Cheeger constant is defined as 
$$h_G=\min_{A\subseteq V, 0<|A|\leq |V|/2} \frac{|\partial A|}{|A|} .$$
Now let $A\subseteq V$ with $0<|A|\leq |V|/2 $. Denote the number of central nodes in $A$ by $a$ and the number of source and target nodes by $b$. By assumption 
$$a+b\leq |V|/2=n+k/2.$$
We now consider three cases:
\textbf{case 1:} If $a\geq k/2$, then necessarily $b\leq n$. It follows that 
$$\frac{|\partial A|}{|A|}\geq \frac{ a\cdot (2n-b)}{a+b}\geq \frac{nk/2}{4n}=\frac{k}{8} $$

\textbf{case 2:} If $b\geq n$ then necessarily $a\leq k/2$. It follows that 
$$\frac{|\partial A|}{|A|}\geq \frac{b \cdot (k-a)}{a+b} \geq \frac{n \cdot k/2}{4n}=\frac{k}{8}. $$

\textbf{case 3:} If both $b\leq n$ and $a \leq k/2 $. Then 
$$\frac{|\partial A|}{|A|}=\frac{a(2n-b)+b(k-a)}{a+b}\geq \frac{na+bk/2}{a+b}\geq k/2.$$

This concludes the proof.
\end{proof}

\section{Additional Experimental Results}
\label{appendix_additional_experiment_results}

To complement the analysis in the main text, here we presents a more fine-grained evaluation of the effect of the hidden-feature dimension on learning accuracy (\cref{fig_model_accuracy_vs_dimension}).

The results follow the same experimental setup as \cref{fig:acc_vs_n_dim} and further illustrate that increasing the number of nodes makes the problem challenging for MPNNs, whereas the Transformer remains robust, with GAT ranking between the two. As shown in the figure, this difficulty is partially mitigated by increasing the hidden-feature dimension, which improves the performance of both MPNNs and GAT but does not close the gap with the Transformer. Overall, these results reinforce that MPNNs are limited by the bottleneck effect, whereas the Transformer remains largely unaffected. \\

\begin{figure}[h]
    \centering
    \includegraphics[width=\linewidth]{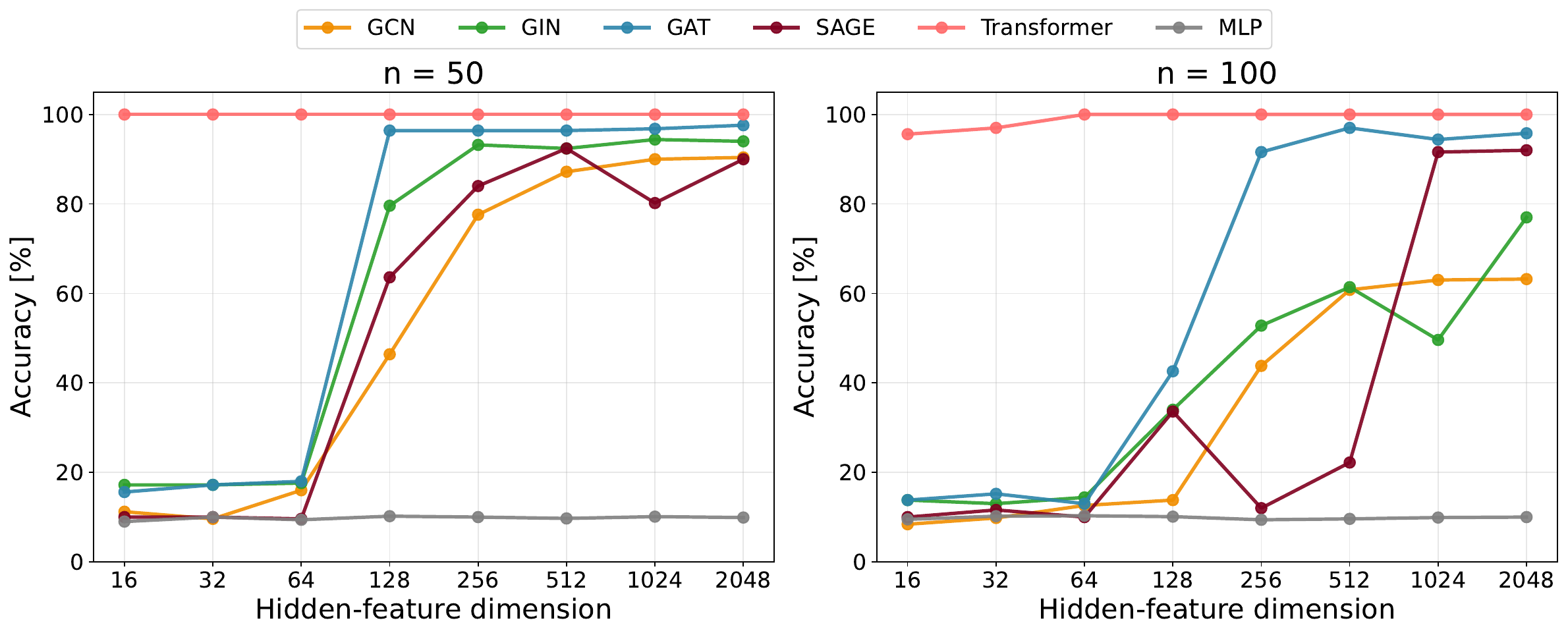}
    \caption{
        Fine-grained evaluation of the effect of hidden-feature dimension
        on learning accuracy in the same experimental setting as
        \cref{fig:acc_vs_n_dim}.
    }
    \label{fig_model_accuracy_vs_dimension}
\end{figure}

\section{Experimental Details}
\label{appendix_experimental_details}
%
%

\subsection*{Input Representation}
The input representation consists of two concatenated one-hot vectors: one encoding the ID and another encoding the label. This combined one-hot representation is fed directly into the initial MPNN layer, which maps it from the input dimensionality to the hidden dimension. Our implementation does not use a separate encoder or projection layer for this operation.

\subsection*{Mean Absolute Difference (MAD)}
In \cref{fig:vanishing} we included a relative variant of the Mean Absolute Difference (MAD)---an energy proposed by \citet{chen2020measuring} to assess the extent to which oversmoothing takes place. For the output feature vectors $h^{K}_{v}$ of the target nodes $v \in T$, it is computed by:
\begin{equation}
\frac{\frac{1}{\tfrac{1}{2} \lvert T \rvert \left( \lvert T \rvert - 1\right)}\sum_{u \neq v \in T} \lVert h^{K}_{u} - h^{K}_{v}\rVert}
{ \frac{1}{\lvert T \rvert} \sum_{v \in T} \lVert h^{K}_{v}\rVert },
\end{equation}
namely, the average pairwise distance between the target node features, divided by the average norm of those features.

\subsection*{Hardware and Software}
All experiments were conducted on an NVIDIA A40 GPU with 48GB memory, using CUDA 12.8. We implemented all models using PyTorch with PyTorch Geometric for graph operations and PyTorch Lightning for training management.
\subsection*{Training procedure}
All models are trained using Adam optimizer with cross-entropy loss, running for a maximum of 1000 epochs with ReduceLROnPlateau learning rate scheduler.

\subsection*{Model Architectures and Initialization}
We evaluated the Two-Radius problem with node configurations $n \in \{10, 50, 100, 150, 200\}$ and central nodes $k=1$ (default), with additional ablation studies using $k \in \{5, 10, 20\}$. For  reproducibility, we fixed the seed to be 0.
\subsection*{Implementation details}
\begin{itemize}
    \item For the Two-Radius problem, we masked source and central nodes during evaluation, computing accuracy only on target nodes.
    \item In gradient norm computation, we use the first 5 test samples to reduce computational cost.
    \item\noindent\textbf{Central nodes (k).} Fixed $n=100$, varied $k \in \{1, 5, 10, 20\}$ using GCN with lr = $5 \times 10^{-4}$. 
\item \textbf{Node Identifiers and Labels.} 
For the Two-Radius problem, nodes are initialized as follows:
\begin{itemize}
    \item All identifiers and labels are encoded as one-hot vectors.
    \item \textit{Source nodes:} Each source node $s \in S$ receives a unique identifier $\id_s$ randomly sampled without replacement from $\{1, \ldots, n\}$, and a label $\ell_s$, chosen randomly from $\{1, \ldots, L\}$ where $L=10$ is the number of classes in our experiments.
    \item \textit{Target nodes:} Each target node $t \in T$ receives a unique identifier $\id_t$ from $\{1, \ldots, n\}$, chosen randomly without replacement. Labels are assigned a constant value $L$, chosen arbitrarily.
    \item \textit{Central nodes:} In the case of one central node $c$, the node was assigned the identifier $\id_c = n+1$. In the case of $k$ central nodes, identifiers were assigned $\{n+1, \ldots, n+k\}$. All central nodes were labeled $\ell_c = 0$.
\end{itemize}

\item\noindent\textbf{Virtual Nodes (VNs).}
To test whether VNs mitigate the short-range bottleneck, we augment GCN with virtual nodes following \citet{cai2023connection} and \citet{hwang2024analysis}. In the single-VN variant, the VN is initialized to zero and, after each layer, a two-layer MLP updates it; the resulting VN embedding is then added (broadcast) to every node representation. In the multiple-VN variant, each VN has its own MLP and is updated independently, and at each layer the sum of all VN embeddings is added to every node. The VN update uses a global aggregation over node features (either mean or sum). We compare GCN with and without VNs using tuned learning rates (5\(\times 10^{-4}\) with VN; \(1\times 10^{-4}\) without).

\item \textbf{Set Transformer.} 
Following \citet{Lee:2019set}, we implement a set-based transformer that \emph{ignores edges} and operates solely on node features,treating the input as a multi-set. This Set Transformer architecture, which can be viewed as a Graph Transformer operating without explicit edge information, processes nodes as an unordered collection where multiplicities matter. In the Two-Radius setting, the model receives one-hot node features (IDs and labels) without connectivity information; thus, source/target/central roles must be inferred from feature patterns. All nodes interact simultaneously via self attention, so any node can attend to all other nodes without routing information through central nodes.

This design is particularly effective here because target nodes can directly attend to sources with matching IDs, bypassing the central-node bottleneck that constrains MPNNs.
\end{itemize}
\subsection*{Hyperparameter Selection}

Table~\ref{tab:model-hyperparams} presents the hyperparameter configurations used for all model architectures in our experiments. These parameters were selected through preliminary experiments that balanced memory constraints, computational runtime, and model performance. Each configuration represents the optimal trade-off between these factors for the respective architecture. Table~\ref{tab:optimal-lr} shows the best-performing learning rates for each model across different hidden dimensions (256 and 1024), determined by evaluating three candidate learning rates per configuration and selecting the one yielding the highest validation accuracy. Together, these hyperparameter choices ensure reproducible and fair comparisons across all evaluated models.

\vspace{+2cm}
\begin{table}[!h]
	\centering
	\caption{Model Hyperparameters}
	\label{tab:model-hyperparams}
        \scriptsize
	\begin{tabular}{lcccccc}
		\toprule
		\textbf{Hyperparams} & \textbf{GCN} & \textbf{GAT} & \textbf{GIN} & \textbf{GraphSAGE} & \textbf{MLP} & \textbf{Set Transformer} \\
		\midrule
		batch\_size & 64 & 32 & 64 & 64 & 64 & 64 \\
		train\_samples & 7000 & 7000 & 7000 & 7000 & 7000 & 7000 \\
		test\_samples & 700 & 700 & 700 & 700 & 700 & 700 \\
		layers & 4 & 4 & 4 & 4 & 4 & 2 \\
		activation & LeakyReLU & LeakyReLU & LeakyReLU & LeakyReLU & ReLU & ReLU \\
		residual & \checkmark & \checkmark & \checkmark & \checkmark & \checkmark & \checkmark \\
		use\_layer\_norm & \checkmark & \checkmark & \checkmark & \checkmark & \checkmark & \checkmark \\
		use\_activation & \checkmark & \checkmark & \checkmark & \checkmark & \checkmark & \checkmark \\
		Attention\_heads & -- & 2 & -- & -- & -- & 2 \\
		lr\_factor & 0.1 & 0.3 & 0.5 & 0.1 & 0.1 & 0.5 \\
		dropout & -- & 0.1 & -- & -- & 0.3 & 0.1 \\
		\bottomrule
	\end{tabular}

\end{table}

\vspace{+1cm}  
\begin{table}[!h]
\centering
\caption{Optimal Learning Rates}
\label{tab:optimal-lr}
\begin{tabular}{lcc}
    \toprule
    \multirow{2.5}{*}{\textbf{Model}} & \multicolumn{2}{c}{\textbf{Hidden Dimension}} \\
    \cmidrule(lr){2-3}
    & \textbf{256} & \textbf{1024} \\
    \midrule
    GCN & $5 \times 10^{-4}$ & $1 \times 10^{-4}$ \\
    GAT & $1 \times 10^{-4}$ & $5 \times 10^{-5}$ \\
    GIN & $5 \times 10^{-5}$ & $5 \times 10^{-4}$ \\
    GraphSAGE & $5 \times 10^{-5}$ & $1 \times 10^{-4}$ \\
    MLP & $1 \times 10^{-4}$ & $1 \times 10^{-4}$ \\
    Set Transformer & $1 \times 10^{-3}$ & $1 \times 10^{-3}$ \\
    \bottomrule
\end{tabular}
\end{table}

\end{document}

%% file: imports.tex
\DeclareMathAlphabet{\mathsfit}{\encodingdefault}{\sfdefault}{m}{sl}
\SetMathAlphabet{\mathsfit}{bold}{\encodingdefault}{\sfdefault}{bx}{n}

\def\gG{{\mathcal{G}}}

\def\X{{\mathbf{X}}}

